\documentclass{article}

\PassOptionsToPackage{numbers, compress}{natbib}

\usepackage[preprint]{arxiv}

\usepackage{booktabs}       
\usepackage{amsfonts}       
\usepackage{amsmath,amssymb}
\usepackage{mathtools}

\usepackage{amsmath,amsfonts,bm}

\def\eqref#1{equation~\ref{#1}}

\def\1{\bm{1}}

\DeclareMathAlphabet{\mathsfit}{\encodingdefault}{\sfdefault}{m}{sl}
\SetMathAlphabet{\mathsfit}{bold}{\encodingdefault}{\sfdefault}{bx}{n}

\usepackage{hyperref}
\usepackage{url}
\usepackage{graphics}
\usepackage{graphicx}
\usepackage{subfigure}
\usepackage{caption}
\usepackage{wrapfig}
\usepackage{amsthm}
\usepackage{multirow}
\usepackage{multicol}
\usepackage[ruled,vlined]{algorithm2e}

\usepackage{microtype}
\usepackage{graphicx}
\usepackage{subfigure}
\usepackage{booktabs} 
\usepackage{multirow}

\usepackage{hyperref}
\usepackage{amsmath}
\usepackage{amssymb}
\usepackage{mathtools}
\usepackage{amsthm}

\usepackage[capitalize,noabbrev]{cleveref}

\theoremstyle{plain}
\newtheorem{theorem}{Theorem}[section]

\theoremstyle{definition}
\newtheorem{definition}[theorem]{Definition}

\theoremstyle{remark}
\newtheorem{remark}[theorem]{Remark}

\title{Purify++: Improving Diffusion-Purification with Advanced Diffusion Models and Control of Randomness}

\author{
  Boya Zhang \\
  Peking University \\
  \And
  Weijian Luo \\
  Peking University \\
  \And
  Zhihua Zhang \\
  Peking University \\
}

\begin{document}
\maketitle

\begin{abstract}
Adversarial attacks can mislead neural network classifiers. The defense against adversarial attacks is important for AI safety. Adversarial purification is a family of approaches that defend adversarial attacks with suitable pre-processing. Diffusion models have been shown to be effective for adversarial purification. Despite their success, many aspects of diffusion purification still remain unexplored. In this paper, we investigate and improve upon three limiting designs of diffusion purification: the use of an improved diffusion model, advanced numerical simulation techniques, and optimal control of randomness. Based on our findings, we propose Purify++, a new diffusion purification algorithm that is now the state-of-the-art purification method against several adversarial attacks. Our work presents a systematic exploration of the limits of diffusion purification methods. 

\end{abstract}

\section{Introduction}
Neural Networks had great successes in multiple applications, e.g. image classification \citep{Krizhevsky2012ImageNetCW,He2016DeepRL}, object detection \citep{Girshick2013RichFH}, semantic segmentation \citep{Shelhamer2014FullyCN} and so on. However, many studies have revealed the phenomenon that neural networks are easy to be cheated by adding human-invisible noise to input data \citep{fgsm,attack2}. If a nature image is added little noise by purpose, the human can still recognize the image, but a neural network classifier may fail to classify correctly. Adding such slight noise to mislead the neural network model is called the adversarial attacks~\citep{attack3,attack5,attack6,attack7}. The approach to erase the adversarial attack is called adversarial defense. In the following paper, we call the name \emph{defense} to represent adversarial defense unless we especially mention it. 

Adversarial purification, also known as pre-processing-based defense, is a major line of defense approach. The central idea of purification is to use some off-the-shelf purification model to pre-process input data before data are fed into the classifier. The main advantage of adversarial purification is that there is no need to retrain the classifier, which is different from adversarial training \citep{advtrain1,pgd,advtrain2,advtrain3,ddpmadvpuri1}. The purification models can be plug-and-played into AI systems with minor modifications. 

Diffusion models~\citep{ddpm,scoresde}, a certain class of generative models, have achieved state-of-the-art performance in applications such as images and audios generation ~\citep{beatgan,diffwave}, text-to-image generation~\citep{glide} and molecule designs~\citep{dpmmolecule}, etc. Besides outstanding generation ability, diffusion models have been shown to be effective to serve as adversarial purification models \citep{diffpure,ddpmadvpuri1,ddpmadvpuri2,guidedddpmadvpuri1,guidedddpmadvpuri2}. However, many design choices for diffusion purification still remains unexplored sufficiently. 

\begin{figure}
\centering
\includegraphics[height=4cm,width=8cm]{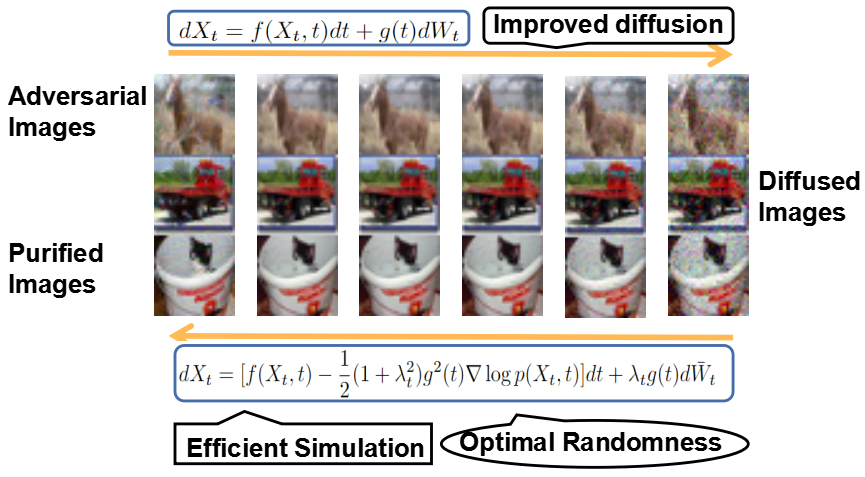}
\caption{\textbf{Illustration of Purify++.} Purify++ improves previous diffusion purification methods in three aspects: (a) improved diffusion model; (b) efficient simulation of purification SDE; (c) optimal control of randomness.}
\label{fig:illustrate}
\end{figure}

In the hope of designing an ideal diffusion purification algorithm that simultaneously has computational efficiency and strong defending performance, we seek to explore and improve the main design aspects of diffusion purification. In subsequent sections, we present our new diffusion purification algorithm, Purify++, which is powered by an improved diffusion model, efficient purification simulations, and optimal control of randomness. We conduct extensive experiments which illustrate that Purify++ achieves the state-of-the-art purification performance on several benchmarking adversarial robustness tests.

\section{Backgrounds}
\subsection{Adversarial Attacks and Defense}
Classifiers that tend to assign wrong predictions on slightly perturbed data are said to be not adversarial robust. To improve the ability of classifiers to resist adversarial attacks, there are mainly two lines of methods for enhancing classifiers' robustness. One line is the adversarial (Adv) training, which consistently trains classifiers with ground truth labels and adversarial attacked data. Such training methods enable the classifier to be non-sensitive to certain attacks and have been reported to be effective for defending various kinds of attacks \citep{advtrain1,pgd,advtrain2,advtrain3}. However, the drawback of Adv training is seeming. Adv training requires operating attacks on input images so it is computationally costly in practice. Besides, Adv trained classifiers can resist only those attacks which the model has been trained with, so the classifier must be re-trained to defend against new attacks. 

\subsection{Diffusion Purification}
Purification, i.e. incorporating a pre-processing procedure before data are fed into classifiers, is another line of methods for improving adversarial robustness. Some existing works have studied the superior defending performance of certain purification models such as denoising auto-encoders~\citep{advpurif-vae,advpurif-dae}, sets of image filters~\citep{deb2020faceguard}, compression and recovery model~\citep{jia2019comdefend} and diffusion models \citep{diffpure}, etc. In this paper, we focus on studies of diffusion models as purification models. 

A diffusion model is constructed as a time-dependent neural mapping $s(x,t)$ that has the same input and output dimensions. The key idea in the design of the diffusion model is to train the neural mapping $s(x,t)$ to match marginal score functions of a forward diffusion SDE \eqref{equ:forwardSDE} with score-matching related techniques ~\citep{Hyvrinen2005EstimationON,scoresde,vincent11,Song2019SlicedSM,Pang2020EfficientLO}. 
\begin{align}\label{equ:forwardSDE}
    dX_t = f(X_t, t)dt + g(t)dW_t, ~p_0=p_d, ~t\in [0,T].
\end{align}
The learned neural mapping $s(x,t)$ is viewed as a good substitution of marginal score functions of \eqref{equ:forwardSDE}.
The idea of diffusion purification aims to pre-process input data with a \emph{forward-and-reverse} procedure of the diffusion model. More precisely, diffusion purification simulates forward diffusion initialized from data to a certain diffusion level which is sufficient to inundate the added attack perturbation. Then it runs a numerical simulation of purification SDE \eqref{equ:rev_sde} or ODE \eqref{equ:rev_ode}, hopping to pull back the diffused data to a point that is close to and have the same output as clean data.\\
\begin{align}\label{equ:rev_sde}
    &dX_t = [f(X_t,t) - g^2(t)s(X_t,t)]dt + g(t)d\Bar{W}_t,\\
    \label{equ:rev_ode}
    &dX_t = [f(X_t,t) - \frac{1}{2}g^2(t)s(X_t,t)]dt, t\in [T,0].
\end{align}
It has been shown in theory that, under certain conditions, the purified data is close to clean data with high probability \citep{diffpure,xiao2022densepure}. By running the forward-and-reverse procedure, adversarial input data can be purified so as to erase the adversarial attacks.

Although previous diffusion purification methods have demonstrated their effectiveness in improving classifiers' adversarial robustness, many aspects of diffusion purification are still not explored sufficiently. In this paper, we analyze and improve some design aspects of diffusion purification, achieving state-of-the-art adversarial robustness on several tests on the CIFAR10 dataset \citep{Krizhevsky2009LearningML}.

\section{Purify++}\label{sec:3}
In this section, we describe three improvements of prior work on diffusion purifications: (1) fast convergence of forward diffusion processes used in previous works can not fully unlock the diffusion model's purification ability; (2) the best sampling techniques for image generation are sub-optimal for purification; (3) the strength of randomness of diffusion purification algorithm should be decided carefully to enhance adversarial robustness. Our proposed algorithm, Purify++, consists of a set of improved design choices: (1) better forward diffusion and corresponding models; (2) advanced sampling techniques; (3) optimal choice of the strength of randomness in the purification procedure.

\subsection{Improved Forward Diffusion}\label{sec:improve_diffusion}
The idea of diffusion purification aims to add noises to adversarial samples by running forward diffusion, and then remove noises by numerical simulation of purification SDE or ODE. A qualified forward diffusion is supposed to be able to simulate forward diffusion cheaply. A common solution for such forward diffusion is the so-called variance preserving (VP) diffusion \citep{scoresde}, whose forward diffusion can be implemented efficiently (with re-parametrization techniques \citep{Kingma2013AutoEncodingVB}) as a scale of data and addition of Gaussian noise.

\subsubsection{Variance Preserving Forward Diffusion}
VP diffusion is the default choice for diffusion purification in prior works \citep{diffpure,xiao2022densepure,guidedddpmadvpuri1,guidedddpmadvpuri2}. The diffusion is defined as
\begin{align}\label{equ:vp_forward}
    dX_t = -\frac{1}{2}\beta(t)X_t dt + \sqrt{\beta(t)}dW_t, t\in [0,T].
\end{align}
Default VP diffusion introduced in \citet{ddpm} takes a linear $\beta$ schedule, i.e. $\beta(t) = 2\beta_2 t^2+ \beta_1 t$. The conditional probability of VP diffusion has explicit form, as a scaling of data sample $x_0$ and injections of Gaussian noise  
\begin{align}\label{eqn:vp_condition}
    X_t|X_0 \sim \mathcal{N}(\sqrt{\alpha_t}X_0, (\mathbf{I}-\alpha_t)\mathbf{I}).
\end{align}
Here $\alpha_t = e^{-\int_0^t \beta(s)ds}$. The conditional mean vanishes as $t\to \infty$ and conditional variance converges to unit variance. The advantage of VP diffusion is the fast convergence rate to the Gaussian distribution when $t\to T$ and $T$ are sufficiently large. More precisely, if $T=\infty$ and $\beta(t)=1$, the VP diffusion turns to the Ornstein-Uhlenbeck process which has in theory exponentially fast convergence rate under certain conditions \citep{Pavliotis2014StochasticPA}. Other studies have shown fast convergence of VP diffusion with more general $\beta$ schedules. 

\subsubsection{Better Diffusion for Purification}
Although VP diffusion is successful in data generation and previous works for purification, we argue that its fast convergence of marginal distribution makes it sub-optimal for adversarial purification tasks. Since diffusion purification only cares about forward-and-reverse procedure within relatively small $t$, the fast convergence to some stationary distribution is unnecessary. So it is unnatural to use VP diffusion for purification. To sidestep this issue, we turn to consider other forward diffusions for purification. 

More precisely, we consider the variance exploding (VE) \citep{scoresde} diffusion and its generalizations \citep{Karras2022ElucidatingTD}. The VE diffusion is defined as 
\begin{align}\label{equ:ve_forward}
    dX_t = \sqrt{\frac{d \sigma(t)^2}{dt}}dW_t, t\in [0,T].
\end{align}
Default VE diffusion takes a log-linear $\sigma$ schedule $\sigma(t)= \sigma_{min}^{(1-t)}\sigma_{max}^t$ where $\sigma_{min}$ and $\sigma_{max}$ are hyper-parameter. The conditional distribution of VE diffusion is 
\begin{align*}
    p_t(x_t|x_0) = \mathcal{N}(x_0, \sigma^2(t)I).
\end{align*}
The variance of conditional distribution explodes to infinity when $t\to \infty$. One key characteristic of VE diffusion is that the conditional mean $\mathbb{E}_{x_t|x_0} x_t = x_0$ does not change with time $t$. While the conditional mean of VP diffusion contracts to zero with coefficient $\alpha_t$ in \eqref{eqn:vp_condition}. This makes two diffused data samples of VP diffusion quickly become indistinguishable with even small $t$. Such property is negative for diffusion purification because purification hopes the forward and reverse procedure is non-sensitive for time steps $t$, so more information on the learned diffusion model can be taken into account for purification instead of those with only relatively small $t$.

\begin{figure}
\label{fig:ve_vp_time_compare}
\centering
\includegraphics[height=3cm,width=8.5cm]{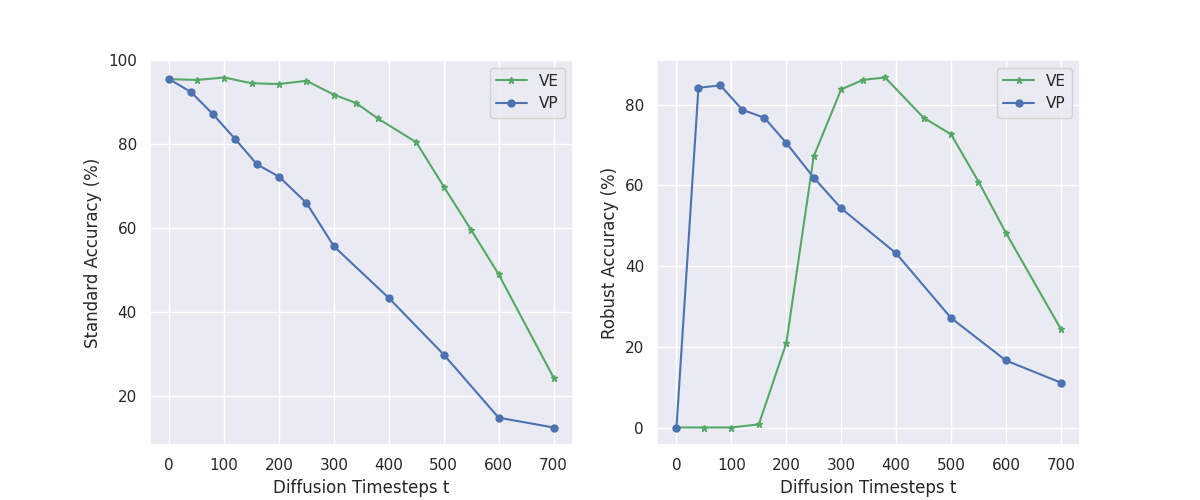}
\caption{
\textbf{Comparison of VE and VP diffusion-based purification test accuracy.}(\textbf{Left}) Standard Accuracy; (\textbf{Right}) Robust Accuracy. The green line represents diffusion purification with VE, the blue line the VP purification. The VE diffusion's optimal time is larger than VP purification which is the same as our proposed Theorem \ref{thm:inter_time} pointed out.}
\end{figure}

To make our argument more concrete, we give a formal qualitative analysis, demonstrating that VP diffusion will have quadratic interaction time while the VE diffusion has linear interaction time for a small data gap. Before our analysis, we formally define the interaction of two Gaussian distributions as
\begin{definition}[Univariate Gaussian Interaction]\label{def:Gaussian_interaction}
For two univariate Gaussian distributions with different mean $\mu_1, \mu_2$ and variance $\sigma_1^2, \sigma_2^2$, we say $\{\mathcal{N}(\mu_i, \sigma_i^2)\}_{i=1,2}$ interacts if the mean and variance satisfy $|\mu_1 - \mu_2| \leq 3\sigma_1 + 3\sigma_2$.
\end{definition}
\begin{remark}
Definition \ref{def:Gaussian_interaction} is inspired by so-called three-$\sigma$ standard which declares that samples from a uni-variate Gaussian distribution $\mathcal{N}(\mu, \sigma^2)$ will locate within interval $[\mu-3\sigma, \mu+3\sigma]$.
\end{remark}
With Definition \ref{def:Gaussian_interaction}, we define the interaction time of two Gaussian process $\mathcal{P}_{t\in \mathcal{T}}$ and $\mathcal{Q}_{t\in\mathcal{T}}$ as the first time when marginal distributions of $\mathcal{P}$ and $\mathcal{Q}$ interact. More concretely, assume $h>0$ and consider two data sample $x_0 = h$ and $x_0 = -h$. Here $2h$ accounts for the gap between the two data. The VP diffusion with linear schedule \eqref{equ:vp_forward} will have a marginal probability 
\begin{align*}
    p_t^{(VP)}(x_t|x_0) = \mathcal{N}(x_t;\sqrt{\alpha_t}x_0, 1-\alpha_t),
\end{align*}
where $\alpha_t = \exp(-\beta_2 t^2 + \beta_1 t)$. $\beta_1$ and $\beta_2$ are quadratic and linear coefficients of the VP schedule. For VE diffusion, the probability takes the form 
\begin{align*}
    p_t^{(VE)}(x_t|x_0) = \mathcal{N}(x_t;x_0, \sigma^2(t)).
\end{align*}
If we consider a relatively small data gap $h$ when $h\to 0$, we show the interaction time changes quadratically with $h$ for VP diffusion and linearly for VE diffusion. We formally give the qualitative result in 
\begin{theorem}[Order of Interaction Time]\label{thm:inter_time}
For VP diffusion with linear $\beta$ schedule \eqref{equ:vp_forward} and VE diffusion with log-linear $\sigma$ schedule \eqref{equ:ve_forward} respectively. The interaction time of VP diffusion is second-order $t^{(VP)}_* \sim o(h^2)$. The interaction time of $t^{(VE)}_* \sim o(h)$.
\end{theorem}
\begin{remark}
Theorem \ref{thm:inter_time} states that for a relatively small data gap, the interaction time of VE diffusion is of the first order, bigger than the interaction time of VP diffusion which is of second order. Such qualitative results confirm the intuition that two distinct data samples of VP diffusion become indistinguishable faster than VE diffusion.
\end{remark}
We put detailed proof of Theorem \ref{thm:inter_time} in Appendix \ref{app:a1}. We conduct an experiment to compare both the optimal robust accuracy and optimal diffusion time of VE and VP-based diffusion purification on the MNIST dataset in Figure \ref{fig:ve_vp_time_compare}. The figure shows firstly that VP's fast convergence harms the overall optimal purification performance. Secondly, relatively few time levels are available for satisfactory purification. However, for VE diffusion, nearly all diffusion time level has similar robustness, which over-perform the best purification performance of VP diffusion. For our proposed Purify++, we use a powerful generalization of VE diffusion, the EDM which is proposed in \citet{Karras2022ElucidatingTD}. 

\subsection{Efficient Purification Simulation}\label{sec:solver}
Numerical simulation techniques (or solvers) are an important design for image generation with diffusion models. Some recent works \citep{JolicoeurMartineau2021GottaGF,Liu2022PseudoNM,Karras2022ElucidatingTD} pointed out that higher-order numerical solvers of SDE \eqref{equ:rev_sde} and ODE \eqref{equ:rev_ode} leads to significant improvements on image generations. But the sacrifice the higher-order solvers pay is the increasing computational costs. Typically, the first-order Euler-Maruyama (EM) discretization will have to $\mathcal{O}(1/h)$ computational cost where $h$ is the discretization step size. However, the higher-order method's computational costs scale polynomially.

Inspired by extensive research on both lower and higher-order numerical methods, we are curious about the question: Do high-order simulation techniques for diffusion purification significantly outperform lower-order ones?

\begin{table}
\caption{Diffusion Purification (Rob Acc) and Image Generation (FID) Performance respect to total diffusion steps on CIFAR10.}
\vskip 0.1in
\centering
\begin{tabular}{l c c c c c }
\toprule
Metric  & \multicolumn{2}{c}{Rob Acc}  &  \multicolumn{2}{c}{FID}  \\
\hline 
Steps   &        EM    &  Heun  &            EM    & Heun \\
\midrule 
25      &        88.04 & 88.07 &            22.4  & 3.7 \\
50      &        88.79 & 88.94 &            6.8   & 3.1 \\
100     &        89.32 & 89.35 &            4.2   & 2.9  \\
\bottomrule
\end{tabular}
\vskip -0.1in
\label{tab:heun}
\end{table}
To get a comprehensive understanding of the difference in choice of numerical methods between diffusion purification and image generation, we conduct an experiment comparing the purification performance of first-order EM and second-order Heun methods with different numbers of diffusion time steps. We also compare the FID, a metric for accessing image generation quality, with the adversarial robustness under the same settings. Table \ref{tab:heun} shows that the generation performance (FID, lower the better) with the EM method explodes dramatically as the number of evaluation steps is reduced. However, the purification performance (Robust Acc) with the EM method does not significantly change. This indicates that techniques for better sampling can not be transferred into diffusion purification arbitrarily. Our findings show that for diffusion purification, first-order solvers perform almost equally well as second-order solvers. We believe one reason is, for sampling, the numerical simulation starts from pure random noise which is easy to introduce much more truncation errors. However, for purification, the simulation starts from slightly noised data, which can potentially give small truncation errors regardless of first or second order.

In practice, we still use Heun's second-order methods for their stability. Besides, we reduce the number of diffusion steps up to 20, which greatly speeds up our proposed diffusion purification algorithm.

\subsection{Optimal Control of Randomness}\label{sec:randomenss}
The third improvement for diffusion purification is the optimal control of randomness. The importance of randomness in defending against adversarial attacks with pre-processing based method has been pointed out in \citet{Cohen2019CertifiedAR,Jia2019CertifiedRF,Rosenfeld2020CertifiedRT}. The randomized smoothing algorithm \citep{Cohen2019CertifiedAR} proposed to add multivariate Gaussian noise on adversarial samples and then feed the noised samples into the classifier. They proved that such addition of random Gaussian noise improves the adversarial robustness of any classifier. 

Besides, we give another view on the positive aspects of randomness in diffusion purification through an extreme example. The ODE \eqref{equ:rev_ode} shares the same marginal distributions as the forward diffusion process \eqref{equ:forwardSDE} as been proved in \citet{scoresde}. If we simulate both the forward and purification procedure with ODE \eqref{equ:rev_ode}, the input data can be restored exactly so that no purification has been done. But if we simulate forward diffusion with SDE \eqref{equ:forwardSDE}, and then the purification procedure with ODE, we find empirically the adversarial robustness is improved. 

Inspired by the above two intuitions, we study the influence of randomness in diffusion purification algorithms. More precisely, we reformulate the purification SDE \eqref{equ:rev_sde} as a combination of purification ODE \eqref{equ:rev_ode} and Langevin dynamics (LD) \citep{Brooks2011HandbookOM}.
\begin{align}
    &dX_t = [f(X_t,t) - \frac{1}{2}g^2(t)\nabla \log p(X_t,t)]dt\\
    & + [-\frac{1}{2}g^2(t)\nabla \log p(X_t,t)dt + g(t)d\Bar{W}_t].
\end{align}
Generalizing this formulation, any linear combination of ODE and LD also shares the same marginal distribution as the forward diffusion process provided that the initial distribution is the same. We formally present this result in Theorem \ref{thm:mix_sde}. 
\begin{theorem}\label{thm:mix_sde}
Assume $f$ and $g$ are under some conditions. Consider the diffusion process 
\[
dX_t = f(X_t,t)dt + g(t)dW_t,
\]
where $W$ is an independent Weiner process. Assume $0\leq \lambda_t\leq 1$ is a positive mix coefficient. Then the process which mixes the reversed ODE and marginal Langevin dynamics shares the same marginal distribution as the forward diffusion
\begin{align}\label{eqn:mix_sde}
    &dX_t =[f(X_t,t) - \frac{1}{2}(1+\lambda_t^2)g^2(t)\nabla \log p(X_t,t)]dt \\
    &+ \lambda_t g(t)d\Bar{W}_t, t\in [T,0]. 
\end{align}
\end{theorem}
The coefficient $\lambda_t$ controls the strength of randomness of diffusion purification. Through experiments, we find that with optimal choice of randomness strength, the purification performance can be further improved by a significant margin. Table \ref{tab:randomness} records the relation of optimal robust accuracy and random strength $\lambda$. The best robust accuracy occurs when taking $\lambda=0.75$ which corresponds to neither the ODE nor the purification SDE. 

Combining an improved diffusion model, efficient simulation purification method, and control of the strength of randomness, we formally propose Purify++ as a diffusion purification algorithm that integrates the EDM diffusion model, Heun's simulation, and control of randomness. 

\begin{table}
\caption{Influence of strength of randomness in Purify++. We choose the strength of randomness ($\lambda$) to vary from 0 to 1 and record corresponding standard accuracy (St Acc) and robust accuracy (Rob Acc). The experiment is conducted on the MNIST dataset.}
\vskip 0.1in
\centering
\small
\scalebox{0.9}{
\begin{tabular}{l c c c c c c}
\toprule
$\lambda$ & 0.0 (ODE) & 0.25 &  0.5 & 0.75 & 1.0 (SDE) \\
\midrule 
Std ACC  & 99.17  & 99.22 & 99.10 & 99.06 & 99.00\\
Rob ACC & 90.51 & 92.53 & 93.15 & \textbf{93.36} & 92.98\\
\bottomrule
\end{tabular}}
\vskip -0.1in
\label{tab:randomness}
\end{table}

\section{Experiments}
In this section, we validate the performance of our proposed method Purify++ on various benchmarks and datasets. We also perform ablation studies to quantify the improvement obtained by the three key components of our method, respectively. Due to the space limitation, we defer experimental details and more ablation study experimental results to Appendix \ref{app:exp1}.

\paragraph{Models} As for the raw classifiers, we consider different architectures for a fair comparison to past works in adversarial defenses. More precisely, for MNIST datasets, we use LeNet~\citep{pgd} model with two convolutional layers. For CIFAR10 datasets, we use Residual Networks~\citep{He2016DeepRL} and Wide ResNet~\citep{wideresnet} with depth 28 and width 10, which are intensely used in prior adversarial defense methods. Input images are normalized according to how the classifiers are trained. As for the diffusion model, we use EDM, a generalization of the VE diffusion model proposed in \citet{Karras2022ElucidatingTD}. We use Heun's second-order numerical simulation of SDE \eqref{equ:rev_sde} and study the trade-off between computational cost (NFEs) and adversarial performance. The strength of randomness is controlled by a mixing coefficient as described in Section \ref{sec:3}.

\paragraph{Adversarial Attacks and Evaluation Metrics} We evaluate our proposed Purify++ against three kinds of current strongest attacks mainly on $\ell_p$-norm bounded threat models and compare the performance with other existing adversarial defense methods. The diffusion model and classifier are trained independently on the same training dataset as in previous works, and the performance metrics are calculated on corresponding test datasets. We leverage two main metrics to evaluate the performance of our defense proposal: \textit{standard} and \textit{robust} accuracy. The standard accuracy measures the performance of the defense method on clean data, which is evaluated on the whole test set. The robust accuracy measures the classification performance on adversarial examples generated by different attacks.

\subsection{Defending against Black-Box Attacks}
We first evaluate our method against black-box attacks, which are constructed with no information about classifiers' detailed compositions~\citep{attack5}. In this setting, attackers know nothing about the internal structure of the model being attacked, so neither the classifier architecture nor the gradient with respect to the loss function. 

SPSA~\citep{Spall1992MultivariateSA, Uesato2018AdversarialRA} is a kind of score-based strong black-box attack, which uses stochastic perturbation-based optimization techniques to search for attacking directions. Since SPSA requires a large number of examples to estimate the gradient, we use 1,280 queries the same as in \citet{scoreadvpuri}. Square Attack~\citep{Andriushchenko2019SquareAA} is a query-efficient black-box attack that integrates a random search scheme for improved efficiency. It has become a standard attacking strategy for evaluations of defenses including both adversarial training and adversarial purification methods.

We conduct experiments on the CIFAR10 dataset under Square Attack and SPSA to evaluate the ability of our proposed method for defending against black-box attacks. To assess the empirical robustness of Purify++, we take into account both $\ell_{\infty}$-norm and $\ell_2$-norm threat models. As is shown in Table \ref{tab:bbattack}, we can see that our method has much higher robust accuracy against both SPSA and Square Attack than prior adversarial purification methods by a large margin. These results show that our method is insensitive to the impact of gradient obfuscation.

\begin{table}[h]
\caption{Robust accuracy against strong black-box attacks: Square Attack and SPSA, on the CIFAR10 dataset, evaluating on WideResNet-28-10. (* marks that the corresponding results are obtained on a subset from the whole test dataset.)}
\label{tab:bbattack}
\vskip 0.1in
\begin{center}
\begin{tabular}{ccccc}
\toprule
Attack & Method  & Robust Acc\\
\midrule
Square Attack & & \\
\hline
\quad $\ell_{\infty}$ & \citet{diffpure} & 85.42* \\
\quad $\ell_{\infty}$ & Purify++ (Ours) & \textbf{91.21*} \\
\quad $\ell_2$ & \citet{diffpure} & 88.02* \\
\quad $\ell_2$ & Purify++ (Ours) & \textbf{91.80*} \\
\midrule
SPSA  \\
\hline
\quad 1280 & \citet{scoreadvpuri} & 80.8\\
\quad 1280 & Purify++ (Ours) & \textbf{90.5*}\\
\bottomrule
\end{tabular}
\end{center}
\vskip -0.1in
\end{table}

\subsection{Defending against Gray-Box Attacks}

In the gray-box setting, attackers have full access to the classifier but have no knowledge about the pre-processor, i.e., the purification model. Adversarial examples are generated only upon the classifier while being evaluated with the whole model. This kind of malicious perturbation is limited, thus is less effective than white-box attacks. Note that gray-box attacks can be considered transfer-based black-box attacks, with the raw classifier as the source model and the whole model (including the preprocessor) as the target model. 

To evaluate our defense method against gray-box attacks, we test on adversarial examples from the projected gradient descent (PGD) attack~\citep{pgd} on the raw classifier. 
PGD is a gradient-based attack, which iteratively perturbs input images along the direction of the gradient of classifier loss function projection to some $l_p \epsilon$-ball.
$$\delta \leftarrow \Pi_\epsilon\left(\delta+\alpha \cdot \operatorname{Proj}\left(\nabla_x \operatorname{Loss}(f(x+\delta), y)\right)\right).$$ 
To verify the consistency of Purify++ across different datasets, we perform experiments against Classifier-PGD attacks for both MNIST and CIFAR10 datasets. Figure \ref{fig:mnist_vis} shows a visualization of Purify++ when defending against Classifier-PGD attacks on the MNIST dataset. It demonstrates that our defense method can purify the attacked images successfully, which leads to higher robust accuracy, after feeding the purified images to pre-trained classifiers regardless of the choice of network architecture.

\begin{figure}
\vskip 0.1in
\centering
\includegraphics[height=3cm,width=6cm]{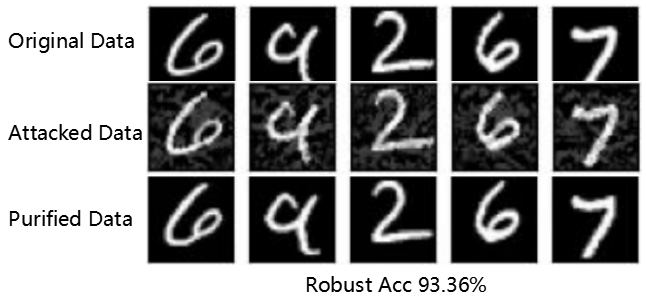}
\caption{Illustration of Purify++ on MNIST for defending against Classifier-PGD attack under $\ell_\infty(\epsilon=0.25)$ threat model.}
\label{fig:mnist_vis}
\vskip -0.1in
\end{figure}

\begin{table}[h]
\caption{Standard accuracy and robust accuracy against Classifier-PGD attack under $\ell_{\infty}(\epsilon=8 / 255)$ threat model on the CIFAR10 dataset, compared with other preprocessor-based adversarial defenses and adversarial training methods against transfer-based attacks, using different classifier architecture.} \label{tab:cifar10-1}
\vskip 0.1in
\begin{center}
\setlength{\tabcolsep}{3pt}
\begin{tabular}{lllc}
\toprule 
\multirow{2}{*}{ Models } & \multicolumn{2}{c}{ Accuracy $(\%)$} & \multirow{2}{*}{Architecture} \\
& Standard & Robust & \\
\midrule
Raw Classifier & 95.19 & 0.00 & WRN28-10 \\
\midrule
\multicolumn{4}{l}{Adversarial training} \\
\citet{pgd} & 87.30 & 70.20 & Res56 \\
\citet{Zhang2019TheoreticallyPT} & 84.90 & 72.20 & Res56 \\
\midrule
\multicolumn{4}{l}{Adversarial purification} \\
\citet{Du2019ImplicitGA} & 48.70 & 37.50 & WRN28-10 \\
\citet{Grathwohl2019YourCI} & 75.50 & 23.80 & WRN28-10 \\
\citet{Yang2019MENetTE} & 94.90 & 82.50 & Res18 \\
\citet{Song2017PixelDefendLG} & & & \\
\hspace{3mm}Natural & 82 & 61 & Res62 \\
\hspace{3mm}+AT & 90 & 70 & Res62 \\
\citet{scoreadvpuri} & & & \\
\hspace{3mm}$\sigma=0.1$ & 93.09 & 85.45 & WRN28-10 \\
\hspace{3mm}$\sigma=0.25$ & 86.14 & 80.24 & WRN28-10 \\
\midrule
 Purify++ & 92.76 & 89.26 & \multirow{2}{*}{ WRN28-10 } \\
(E+H) &  \begin{tiny}$\pm 0.16$\end{tiny} &  \begin{tiny}$\pm 0.16$\end{tiny} & \\
 Purify++ & 92.41 & \textbf{90.08} & \multirow{2}{*}{ WRN28-10 } \\
(E+H+R) &  \begin{tiny}$\pm 0.14$\end{tiny} &  \begin{tiny}$\pm 0.07$\end{tiny} & \\
\bottomrule
\end{tabular}
\end{center}
\vskip -0.1in
\end{table}

Table \ref{tab:cifar10-1} presents the comparison of adversarial performances on the CIFAR10 dataset under $\ell_\infty(\epsilon=8/255)$ threat model, and Table \ref{tab:cifar10-2} shows the robustness performance under $\ell_2(\epsilon=0.5)$. The experimental settings follow previous purification methods closely, including the attack budgets and input data normalization. 

We compare our method with two kinds of defenses: adversarial purification methods and adversarial training methods. As Table \ref{tab:cifar10-1} indicates, our method obtains high robust accuracy, and outperforms previous state-of-the-art methods, on CIFAR10 against Classifier-PGD ($l_\infty$-norm) attack, with comparable standard accuracy. Table \ref{tab:cifar10-1} also demonstrates an ablation study of our proposed three improvements: improved diffusion model, advanced simulation technique, and optimal randomness control. Each of the improvements contributes to improved robust accuracy. Table \ref{tab:cifar10-2} shows the robustness performance against Classifier-PGD ($l_2$-norm) attack. Purify++ achieves higher robust accuracy than previous defenses. 

AutoAttack~\citep{Croce2020ReliableEO}, an ensemble of powerful parameter-free attacks, has become a standard adversarial robustness benchmark for evaluating adversarial defense strategies with high reliability. For a defense strategy, AutoAttack's standard version composes of a series of challenging attacks including variants of PGD, FAB \citep{Croce2019MinimallyDA}, and Square-Attack. For a comprehensive evaluation of Purify++, we conduct a comparison with the previous SOTA diffusion purification method, DiffPure~\citep{diffpure}. We use the standard version of AutoAttack and transfer the attack from the raw classifier to our purification model, keeping the same experimental settings with \citet{diffpure}. 
\begin{table}[h]
\caption{Standard accuracy and robust accuracy against Classifier-PGD attack under $\ell_2(\epsilon=0.5)$ threat model on CIFAR10 dataset, compared with other preprocessor-based adversarial defenses and adversarial training methods.} \label{tab:cifar10-2}
\vskip 0.1in
\begin{center}
\setlength{\tabcolsep}{4pt}
\begin{tabular}{lccl}
\toprule 
\multirow{2}{*}{ Models } & \multicolumn{2}{c}{ Accuracy $(\%)$} \\
& Standard & Robust \\
\midrule
Raw Classifier & $95.19$ & $0.30$ &  \\
\midrule
\citet{Rony2019} & 89.05 & 67.6 \\
\citet{Ding2020MMA} & 88.02 & 66.18 \\
\citet{Rice2020OverfittingIA} & 88.67 & 71.6 \\
\citet{wu2020adversarial} & 88.51 & 73.66 \\
\citet{Gowal2020UncoveringTL} & 90.9 & 74.5 \\
\midrule
Purify++ & $\mathbf{92.29} \pm 0.20$ & $\mathbf{91.13} \pm 0.18$ &  \\
\bottomrule
\end{tabular}
\end{center}
\vskip -0.1in
\end{table} 
Table \ref{tab:cifar10-3} summarizes the main results. We compare the robust accuracy with the AutoAttack standard attacks under $l_\infty$ and $l_2$ norm. The Purify++ outperforms DiffPure under both $l_\infty$ and $l_2$ norm threat models. Combining results in Table \ref{tab:cifar10-1}-\ref{tab:cifar10-3}, we can conclude that Purify++ consistently achieves better robustness performance against transfer-based attacks than previous defenses. 

\begin{table}[h]
\caption{Comparison with DiffPure against transfer-based attacks under different $\ell_p$-norm threat models.} \label{tab:cifar10-3}
\vskip 0.1in
\begin{center}
\begin{tabular}{ccccc}
\toprule
$\ell_p$-norm & Classifier & DiffPure & Purify++ \\
\midrule
$\ell_{\infty}$ & $0.00$ & $89.58$ & $\mathbf{91.29 \pm 0.23}$ \\
$\ell_2$ & $0.00$ & $90.37$ &  $\mathbf{91.28 \pm 0.38}$\\
\bottomrule
\end{tabular}
\end{center}
\vskip -0.1in
\end{table}

\subsection{Defending against Strong Adaptive Attacks}

\begin{table}[h]
\caption{Standard accuracy and robust accuracy against BPDA+EOT attack under $\ell_{\infty}(\epsilon=8 / 255)$ threat model on CIFAR10, compared with other preprocessor-based adversarial defenses and adversarial training methods against white-box attacks, using different classifier architecture.} \label{tab:cifar10-4}
\vspace{0.1in}
\centering
\setlength{\tabcolsep}{4pt}
\begin{tabular}{llllc}
\toprule
\multirow{2}{*}{ Models } & \multicolumn{2}{c}{Accuracy (\%)} & \multirow{2}{*}{Architecture} \\
 & Natural & Robust & \\
\midrule
\multicolumn{4}{l}{Adversarial training} \\
\citet{pgd} & 87.3 & 45.80 & Res18 \\
\citet{Zhang2019TheoreticallyPT}  & 84.90 & 56.43 & Res18 \\
\citet{Carmon2019UnlabeledDI} & 89.67 & 63.10 & WRN28-10 \\
\citet{Gowal2020UncoveringTL} & 89.48 & 64.08 & WRN28-10 \\
\midrule
\multicolumn{4}{l}{Adversarial purification} \\
\citet{Song2017PixelDefendLG} & 95.00 & 9.00 & Res62 \\
\citet{Yang2019MENetTE} & & &\\
\hspace{3mm}BPDA & 94.80 & 40.80 & Res18 \\
\hspace{3mm}Approx. Input & 89.40 & 41.50 & Res18 \\
\hspace{3mm}Approx. Input (+AT) & 88.70 & 62.50 & Res18 \\
\citet{Hill2020StochasticSA} & 84.12 & 54.90 & WRN28-10 \\
\citet{scoreadvpuri} & 86.14 & 70.01 & WRN28-10 \\
\citet{diffpure} & 89.02 & 81.40 & WRN28-10 \\
\midrule
Purify++ & 89.91 & \textbf{83.33} & \multirow{2}{*}{ WRN28-10 } \\
(Ours) & \begin{tiny}$\pm 0.51$\end{tiny} & \begin{tiny}$\pm 1.18$\end{tiny} & \\
\bottomrule
\end{tabular}
\vspace{-0.1in}
\end{table}

\begin{table}[!h]
\caption{\textbf{Insensitivity Property of Purify++.} The table records the best robust accuracy and corresponding standard accuracy of Purify++ on the CIFAR10 test dataset with different total diffusion steps against Classifier-PGD attacks under $\ell_{\infty}$-norm threat model.}
\label{tab:non-sensitivity-pgd}
\vskip 0.1in
\begin{scriptsize}
\begin{center}
\begin{tabular}{ccccc}
\toprule
Total Diffusion Steps $T$ &  Diffusion Steps for Purification $t^*$ & \# of NFEs & Standard Accuracy & Robust Accuracy\\
\midrule
25 & 8 & 16 & 91.72 & 88.07(-\%1.43)\\
50 & 16 & 32 & 91.54 & 88.94(-\%0.46) \\
100 & 30 & 60 & 92.77 & \textbf{89.35}\\
200 & 60 & 120 & 92.81 & 89.06(-\%0.32)\\
\bottomrule
\end{tabular}
\end{center}
\end{scriptsize}
\vskip -0.1in
\end{table}

In this section, we further validate our defense method against strong adaptive attacks. Since the diffusion model involves multiple iterations through neural networks, it might lead to issues with obfuscated gradients. BPDA~\citep{bpda} is specially designed for non-differentiable preprocessing-based defense strategies. The idea of BPDA aims to memorize the forward pass of integration of the pre-processer and classifier to approximate the derivative w.r.t. input sample. Besides, taking the stochastic defense methods into account, the Expectation over Transformation (EOT)~\citep{attack7} is a method for calculating the gradients of random functions as the expectation values of multiple random stochastic sessions. The adaptive attack, BPDA+EOT, is commonly used for evaluating adversarial purification methods with stochasticity. 

We conduct experiments against BPDA+EDT attack under $\ell_{\infty}(\epsilon=8 / 255)$ threat model on CIFAR10 dataset. Note that, we assess the robust accuracy of our method on a fixed subset of 512 pictures randomly picked from the test set due to the high computational cost of applying strong adaptive attacks to our whole model, the same as \citet{diffpure}. Table \ref{tab:cifar10-4} shows the standard and robust accuracy performance compared with other adversarial defenses. Our method achieves the best robust accuracy while maintaining high standard accuracy, outperforming prior diffusion purification methods.

\subsection{Insensitivity to Total Sampling Steps}
Besides the standard adversarial robustness that we have shown in previous sections, Purify++ has an additional advantage, i.e. insensitivity to the total diffusion steps of the purification procedure. The robustness performance does not drop dramatically with the reduction of total diffusion steps. Such insensitivity property is important for the purification algorithm in practice because we can achieve significant efficiency improvement at a small cost of robust accuracy. In this section, we further study the insensitivity property of Purify++ with the optimal control of randomness. Table \ref{tab:non-sensitivity-pgd} shows the Purify++'s robust and standard accuracy under different total diffusion steps (with different numbers of NFEs). For Classifier PGD attacks, the robust accuracy drops less than $\%1.5$ when the number of total diffusion steps varies from 200 to 25 (the number of NFEs from 120 to 16). This key property indicates that Purify++ is reliable under different levels of computation budgets.

\section{Related Works}
\subsection{Adversarial Purification}
Adversarial purification is a paradigm for defending against attacks by pre-processing the input data before data are fed into the classifier.\citet{scoreadvpuri} used a pre-trained NCSN model \citep{ncsnv1} and simulated Langevin dynamics on noised data to purify attacked samples. \citet{diffpure} formally used a pre-trained diffusion model to purify adversarial samples. They simulated VP forward diffusion initialized with attacked samples and ran default sampling SDE of the VP diffusion model for purification. They did not use advanced simulation techniques and improved forward diffusion. \citet{guidedddpmadvpuri1} and \citet{guidedddpmadvpuri2} also used modified diffusion models' sampling SDE for purification in terms of robust accuracy on testing data on the CIFAR10 dataset. \citet{carlini2022certified} and \citet{xiao2022densepure} studied the empirical evaluation of the certified robustness of diffusion purification.

\subsection{Improvements on Diffusion Models}
Improvements in both architectures and sampling techniques of diffusion models are intensively studied. \citet{scoresde} used UNet architectures for diffusion models. \citet{Kingma2021VariationalDM} additionally introduced Fourier features into UNet architectures. \citet{Dockhorn2021ScoreBasedGM} replaced the forward diffusion model with critically damped Langevin dynamics. \citet{Peebles2022ScalableDM} replaced UNet model with vision transformer model \citep{Dosovitskiy2020AnII} which is more scalabal to data dimensions. \citet{Karras2022ElucidatingTD} proposed a generalization of VE forward diffusion together with training techniques such as data augmentations. 

High-performing numerical simulation of sampling procedure of diffusion models is another hot research topic \citep{JolicoeurMartineau2021GottaGF,Liu2022PseudoNM,Lu2022DPMSolverAF,Luhman2021KnowledgeDI,improveddiff,Salimans2022ProgressiveDF,Watson2021LearningTE,Watson2022LearningFS,Karras2022ElucidatingTD}. Among them, studies on higher order numerical solution of SDE or ODE related to Purify++ the most. \citet{Liu2022PseudoNM} and \citet{Zhang2022FastSO} proposed to use of linear multi-step methods which accelerate the sampling efficiency of the diffusion model. \citet{Lu2022DPMSolverAF} pointed out the semi-linear structure of VP diffusion and propose a DPM-Solver with a smaller numerical error. \citet{Karras2022ElucidatingTD} and \citet{JolicoeurMartineau2021GottaGF} studied the use of Heun's method for the simulation of ODE. 

\section{Conclusion}
The paper explores three design choices for diffusion purification methods and demonstrates the benefits of improved designs in terms of empirical robustness. The proposed Purify++ achieves state-of-the-art diffusion purification performance on the CIFAR10 dataset. Our proposed Purify++ not only tries to torch the limits of diffusion purification methods but also sheds some light on further developments in diffusion-based purification methods.

\newpage
\bibliography{biography}
\bibliographystyle{IEEEtranN}

\newpage
\appendix
\onecolumn
\section{Proofs in Section \ref{sec:3}}
\subsection{Proof of Theorem \ref{thm:inter_time}}\label{app:a1}
\begin{theorem}[Order of Interaction Time]
Consider VP diffusion with linear $\beta$-schedule \eqref{equ:vp_forward} and VE diffusion with log-linear $\sigma$-schedule \eqref{equ:ve_forward} respectively. The interaction time of VP diffusion is second-order $t^{(VP)}_* \sim o(h^2)$. The interaction time of $t^{(VE)}_* \sim o(h)$.
\end{theorem}
\begin{proof}
We divided the proof for VP diffusion and VE diffusion separately. 
\paragraph{VP interaction time} By definition of VP diffusion with linear $\beta$-schedule \ref{equ:vp_forward}. The conditional distribution writes
\begin{align*}
    p_t^{(VP)}(x_t|x_0) = \mathcal{N}(x_t;\sqrt{\alpha_t}x_0, 1-\alpha_t).
\end{align*}
Here $\alpha_t = \exp(-\int_0^t \beta(s)ds) = \exp(-\beta_2 t^2 -\beta_1 t)$. So initializing from two distinct point $h$ and $-h$, marginal distributions write 
\begin{align*}
    &p_t(x_t) = \mathcal{N}(x_t;e^{-\beta_2 t^2 -\beta_1 t}h, 1-e^{-\beta_2 t^2 -\beta_1 t}),\\
    &q_t(x_t) = \mathcal{N}(x_t;-e^{-\beta_2 t^2 -\beta_1 t}h, 1-e^{-\beta_2 t^2 -\beta_1 t}).
\end{align*}
By definition of interaction time, we have 
\begin{align*}
    t^{(VP)}_* &= \inf_t \{e^{-\beta_2 t^2 -\beta_1 t}h = 3 (1-e^{-\beta_2 t^2 -\beta_1 t})\}.
\end{align*}
So $t^{(VP)}_*$ satisfies equation
\begin{align*}
    e^{-\beta_2 t_*^2 -\beta_1 t_*}h = 3 (1-e^{-\beta_2 t_*^2 -\beta_1 t_*}).
\end{align*}
Reversing the above equation, we conclude the interaction time to be 
\begin{align*}
t^{(VP)}_* &= \frac{1}{2}\frac{\beta_1}{\beta_2} (\sqrt{1+4\frac{\beta_2}{\beta_1}\log (1+\frac{h^2}{9}) } - 1)\\
&\sim \frac{1}{2}\frac{\beta_1}{\beta_2}\big[\frac{4+\frac{\beta_2}{\beta_1}\log (1+\frac{h^2}{9})}{2} \big]\\
&\sim \frac{\beta_1}{\beta_2}\frac{\beta_2}{\beta_1} \log (1+\frac{h^2}{9})\\
&\sim \frac{h^2}{9}\sim o(h^2).
\end{align*}

\paragraph{VE interaction time} By definition of VE diffusion with log-linear $\sigma$-schedule \ref{equ:ve_forward}. The conditional distribution writes 
\begin{align*}
    p_t^{(VE)}(x_t|x_0) = \mathcal{N}(x_t;x_0,\sigma^2(t)).
\end{align*}
Here $\sigma(t) = \sigma_{min}^{(1-t)}\sigma_{max}^t$. So initializing from two distinct point $h$ and $-h$, marginal distribution write
\begin{align*}
    & p_t(x_t)=\mathcal{N}(x_t;h, \sigma_{min}^{2(1-t)}\sigma_{max}^{2t}), \\
    & p_t(x_t)=\mathcal{N}(x_t;-h, \sigma_{min}^{2(1-t)}\sigma_{max}^{2t}).
\end{align*}
By definition of interaction time, we have
\begin{align*}
    t^{(VE)}_* = \inf_t \{h=3 \sigma_{min}^{(1-t)}\sigma_{max}^{t}\}.
\end{align*}
For VE diffusion with log-linear $\sigma$-schedule, $\sigma_{min}$ is often set to be sufficiently small such that $h/3 > \sigma_{min}$ holds. Reversing the above equation, we conclude the interaction time to behold
\begin{align*}
    t^{(VE)}_* &= \log(\frac{h/3}{\sigma_{min}})/\log(\frac{\sigma_{max}}{\sigma_{min}})\\
    &\sim \log(1 + \frac{h/3-\sigma_{min}}{\sigma_{min}})/\log(\frac{\sigma_{max}}{\sigma_{min}})\\
    & \sim (\frac{h/3-\sigma_{min}}{\sigma_{min}})/\log(\frac{\sigma_{max}}{\sigma_{min}})\\
    & \sim o(h).
\end{align*}
\end{proof}

\subsection{Proof of Theorem \ref{thm:mix_sde}}\label{app:a2}
\begin{theorem}
Assume $f$ and $g$ are under some conditions. Consider the diffusion process 
\[
dX_t = f(X_t,t)dt + g(t)dW_t,
\]
where $W$ is an independent Weiner process. Assume $0\leq \lambda_t\leq 1$ is a positive mix coefficient. Then the process which mixes the reversed ODE and marginal Langevin dynamics shares the same marginal distribution as the forward diffusion
\begin{align*}
    &dX_t =[f(X_t,t) - \frac{1}{2}(1+\lambda_t^2)g^2(t)\nabla \log p(X_t,t)]dt + \lambda_t g(t)d\Bar{W}_t, t\in [T,0].
\end{align*}
\end{theorem}

\begin{proof}
Consider a general diffusion \eqref{equ:forwardSDE}, the marginal density functions satisfy the so-called Fokker-Planck equation:
\begin{align}\label{eqn:fp_equation}
    \frac{d}{dt}p(x,t) = -\langle \nabla_x, p(x,t) f(x,t) \rangle + \frac{1}{2}g^2(t)\Delta_x p(x,t).
\end{align}
Notice that 
$\Delta_x p(x,t) = \langle \nabla_x, \nabla_x p(x,t) \rangle = \langle \nabla_x, p(x,t)\nabla_x\log p(x,t) \rangle$, the Fokker-Planck equation \eqref{eqn:fp_equation} turns to 
\begin{align}\label{eqn:fp_eq1}
    \frac{d}{dt}p(x,t) = \langle \nabla_x, p(x,t)\big[ \frac{1}{2}g^2(t)\nabla_x \log p(x,t) -  f(x,t)  \big] \rangle.
\end{align}
Equation \eqref{eqn:fp_eq1} describes the evaluation of the marginal density of forward diffusion \eqref{equ:forwardSDE} and corresponds to an ODE realization of the form
\begin{align*}
    dX_t = \big[f(X_t,t) - \frac{1}{2}g^2(t)\nabla_x \log p(X_t,t) \big]dt, t\in [T,0].
\end{align*}
Let $p(x,t)$ represent the marginal density of forward SDE \eqref{equ:forwardSDE} initialized with data distribution. Since the ODE is reversible, reverting above ODE initializing from $q(x,0)=p(x, T)$ gives
\begin{align}\label{eqn:rev_ode1}
    dX_t = [\frac{1}{2}g^2(T-t)\nabla_{X_t} \log p(X_t,T-t) - f(X_t,T-t)]dt, t\in [0,T].
\end{align}
Let $q(x,t)$ represents the marginal density of reverse ODE \eqref{eqn:rev_ode1}, and $p(x,0) = q(x,T)$. The reversibility of ODE results $p(x,t) = q(x,T-t)$ for $\forall t \in [0,T]$. The Fokker-Planck equation of \eqref{eqn:rev_ode1} writes
\begin{tiny}
\begin{align}
    &\frac{d}{dt}q(x,t) = \langle \nabla_x, q(x,t)\big[f(x,T-t) - \frac{1}{2}g^2(T-t) \nabla_x \log p(x,T-t) \big] \rangle\\
    &= \langle \nabla_x, q(x,t)\big[f(x,T-t) - \frac{1}{2}g^2(T-t) \nabla_x \log p(x,T-t) - \frac{1}{2}\lambda_t^2 g^2(T-t) \nabla_x \log q(x,t) \big] \rangle + \frac{1}{2}\lambda_t^2 g^2(T-t) \Delta_x q(x,t) \\
    & = \langle \nabla_x, q(x,t)\big[f(x,T-t) - \frac{1}{2}g^2(T-t)( \nabla_x \log p(x,T-t) + \lambda_t^2 \nabla_x \log q(x,t)) \big] \rangle + \frac{1}{2}\lambda_t^2 g^2(T-t) \Delta_x q(x,t)\\
    & = \langle \nabla_x, q(x,t)\big[f(x,T-t) - \frac{1}{2}g^2(T-t)(1+\lambda_t^2 )\nabla_x \log p(x,T-t) \big] \rangle + \frac{1}{2}\lambda_t^2 g^2(T-t) \Delta_x q(x,t).
\end{align}
\end{tiny}
The last equation corresponds to a reversed SDE 
\begin{align}
    dX_t = [\frac{1}{2}g^2(T-t)(1+\lambda_t^2)\nabla_{X_t} \log p(X_t,T-t) - f(X_t,T-t)]dt + \lambda_t g(T-t)d\Bar{W}_t, t\in [0,T].
\end{align}
which is the reversing form of SDE in Theorem \ref{thm:mix_sde}.
\end{proof}

\section{Experimental Details}\label{app:exp1}
This section represents details of our experiments in this paper. 

\subsection{Details on Experiments in Section \ref{sec:improve_diffusion}}
In the experiments of Figure \ref{fig:ve_vp_time_compare}, we train a VP diffusion and a VE diffusion model on the MNIST dataset separately. The MNIST data is resized to 32x32 instead of 28x28 with a bilinear interpolation algorithm. 

The diffusion architecture remains the same as is used in \cite{ncsnv2}. The VP diffusion process uses the linear $\beta$ schedule as in \citet{ddpm}, while the VE diffusion uses the log-linear $\sigma$ schedule as in \citet{scoresde}. Both diffusion processes' time is discretized to 1000 steps and the training uses the maximum likelihood weighted denoising score matching objective which is proposed in \citet{scoreflow}. 

We train a wide resent classifier with width factor 16 and depth factor 8. The classifiers architecture closely remains the same as \citet{wideresnet} with a dropout rate of 0.3. The training of the classifier uses the SGD algorithm with a learning rate of 0.1. Both the classifier and diffusion models are trained on training data of the MNIST dataset. 

With trained diffusion models and classifier models, we test both standard accuracy and robust accuracy for discrete time steps, 0,10,...,900, and 1000, for adversarial purification. The standard accuracy accounts for diffusion purification with non-attacked data. The robust accuracy accounts for the test accuracy of diffusion-purified samples of adversarial attacked data. 

\subsection{Details on Experiments in Section \ref{sec:solver}}
Table \ref{tab:heun} compares the different behaviors of numerical solvers with the same diffusion models for diffusion purification and FID evaluation. 

In this experiment, we use a pre-trained diffusion model from \citet{karras22} (\url{https://github.com/NVlabs/edm}). The diffusion model is pre-trained on training data of the CIFAR10 dataset with the \emph{ddpmm++} architecture as proposed in \citet{karras22}. We evaluate the robust accuracy of diffusion purification under PGD ($\ell_\infty$) attack with 40 attack steps. The implementation of the Heun numerical method is modified from EDM's second-order sampler. The Euler-Maruyama scheme is implemented by masking the second-order correction steps of the Heun method. The left two columns record the optimal robust accuracy of EM and Heun methods with the searched optimal diffusion time. The results show that the robust accuracy does not differ much regardless of discretization methods.

The Frechet Inception Distance (FID) is a widely used positive-valued metric for assessing image generation quality. The lower the FID is, the better the sampling quality is. The right two columns list the FID values of samples generated from the same diffusion model as used in purifications. The FIDs are obtained from Figure 2 of \citet{karras22}. The values show that the FIDs differ very much especially when the number of steps is small.

\subsection{Details on Experiments in Section \ref{sec:randomenss}}
Table \ref{tab:randomness} shows the influences of the strength of randomness on purification performance as we formulated in Theorem \ref{thm:mix_sde}. 

To simplify the method, we use a constant $\lambda$ instead of time-dependent $\lambda_t$ as in expression \eqref{eqn:mix_sde}. The diffusion model (VP model) and classifier are the same as the one used in Table \ref{fig:ve_vp_time_compare}. We let the strength of the randomness coefficient $\lambda$ to vary from 0.0, which recovers the deterministic ODE \eqref{equ:rev_ode}, to 1.0, which recovers the default purification SDE \eqref{equ:rev_sde}. We find that the optimal purification performance does not occur when $\lambda=0$ (ODE) or $\lambda=1$ (SDE). The optimal robust accuracy occurs when $\lambda=0.75$, leading to an optimal robust accuracy of 93.36\%. 

The experiment shows that by assigning suitable strength of randomness of diffusion purification, the method's robustness gets improved. 

\subsection{Details on Experiments in Table \ref{tab:bbattack}}
In Table \ref{tab:bbattack} we evaluated the proposed Purify++ for defending two standard black-box attacks: the square attack and the SPSA attack. 

We use the pre-trained diffusion model trained unconditionally on the CIFAR10 dataset from \citet{karras22}. The model architecture is \emph{ddpm++}. We train a wide resnet classifier model with a width factor of 28 and depth of 10, following the same python code as used in \citet{diffpure}. 

For the square attack, we use the auto-attack python library ({\url{https://github.com/fra31/auto-attack}}) with default hyper-parameters, as is also used in previous works we list in the table. For the SPSA attack, we put detailed hyper-parameters in Table \ref{tab:spsa_hypers}.

\begin{table}
\caption{Hyper-parameters of SPSA attacks.}
\vskip 0.1in
\centering
\small
\scalebox{1.0}{
\begin{tabular}{lc}
\toprule
$\lambda$ &  \\
\midrule 
norm  & $\ell_\infty$ \\
$\epsilon$  & 8/255 \\
$\delta$  & 0.01 \\
lr  & 0.01 \\
nb-iter  & 100 \\
nb-sample  & 128 \\
max-batch-size  & 1280 \\
targeted  & False \\
loss-fn  & None \\
clip-min  & 0.0 \\
clip-max  & 1.0 \\
\bottomrule
\end{tabular}}
\vskip -0.1in
\label{tab:spsa_hypers}
\end{table}

\subsection{Details on Experiments in Table \ref{tab:cifar10-1} and Table \ref{tab:cifar10-2}}
Table \ref{tab:cifar10-1} shows a comparison of the proposed Purify++ against other defenses under PGD attacks with $\ell_\infty$ and $\epsilon=8/255$. The PGD step is 40 as is usually used in \citet{scoreadvpuri}. We use the same model and classifier as in Table \ref{tab:bbattack}. We take 5 random seeds for each value and take the average as the posted one. The standard deviation is calculated within 5 random results. 

Especially, we conduct an ablation study in Table \ref{tab:cifar10-1}. The \textbf{Purify++ (E+H)} represents the method with an improved diffusion model (\textbf{E} represents the used EDM model in \citet{Karras2022ElucidatingTD}). The letter H represents the proposed Heun solver for purification SDE. The letter R represents the introduction of control of randomness as described in Section \ref{sec:randomenss}. 

As for Table \ref{tab:cifar10-2}, we test the PGD with norm $\ell_2$ and $\epsilon=128/255$.

\subsection{Details on Experiments in Table \ref{tab:cifar10-3}}
Table \ref{tab:cifar10-3} conducts a special comparison of our proposed Purify++ and a previous diffusion purification method, the Diffpure \citep{diffpure}. 

The DiffPure is a diffusion purification method that used iDDPM (\url{https://github.com/openai/improved-diffusion}), Euler-Maruyama solver and standard diffusion sampling SDE. Our proposed Purify++ improves all three aspects of DiffPure, so there is a necessity of comparing Purify++ and Diffpure on the same benchmarking test. More precisely, we compare the robust accuracy of CIFAR10 under PGD attacks with both $\ell_\infty$ and $\ell_2$ norm. For both attacks, we use the same setting as is used in \citet{diffpure}. 

The experiments show that Purify++ outperforms Diffpure's robust accuracy, which indicates that the improvements that Purify++ takes contribute to better adversarial robustness. 

\subsection{Details on Experiments in Table \ref{tab:cifar10-4}}
Table \ref{tab:cifar10-4} summarizes the empirical robustness of Purify++ for defending the BPDA+EOT attack, a strong white-box attack that is specially designed for pre-processing-based defenses. We closely follow the same setting for BPDA+EOT attacks from \citet{diffpure}. Since the BPDA is very computationally expensive, we evaluate the empirical robustness with a sub-set of images from CIFAR10 of the set size of 512, which is the same as \citet{diffpure}. We also run 3 random seeds and average the robust accuracy to get final outputs. The standard deviation is computed within 3 trials. 

As Table \ref{tab:cifar10-4} shows, Purify++ outperforms the previous state-of-the-art robust accuracy.

\subsection{Details on Experiments in Table \ref{tab:non-sensitivity-pgd}}
Table \ref{tab:non-sensitivity-pgd} records the robust accuracy of Purify++ with different numbers of diffusion's total steps  on the CIFAR10 dataset. The first column, the total diffusion steps, represents the number of diffusion levels that we search for diffusion purification. The second column, the number of diffusion steps for purification, is the optimal diffusion level which results in the best robust accuracy. The third column, the number of neural function evaluations of diffusion purification, is the number of times that we need to execute a forward pass of the diffusion model's network. For the first-order numerical solver, NFEs are supposed to be equal to diffusion steps (the second column). For the second-order numerical solver, NFEs are supposed to be equal to two times of diffusion steps. 

In this experiment, we use PGD with $\ell_\infty$ norm and $\epsilon=8/255$. We use 40 attacking steps as is widely used for defending PGD attacks. The result in Table \ref{tab:non-sensitivity-pgd} shows that Purify++ is non-sensitive to either the total number of diffusion steps or the number of NFEs. This non-sensitive property makes Purify++ both efficient and stable for use in practice. 

\section{More Experimental Results}

\subsection{Ablation study}

\subsubsection{Impact of number of PGD iteration steps}

\begin{table}[h]
\caption{Robust Accuracy against Classifier-PGD attack under $\ell_{\infty}(\epsilon=8 / 255)$ threat model on CIFAR10 dataset, with different PGD iteration steps.}
\centering
\begin{tabular}{c c c c c}
\toprule
PGD Iteration steps & 10 & 20 & 40 & 100 \\
\midrule
Robust Accuracy $(\%)$ & 89.89 & 90.03 & 90.03 & 90.05\\
\bottomrule
\end{tabular}
\label{tab:pgd_iter}
\end{table}

In this section, we present robustness performance of our proposed method in CIFAR10 dataset, using different PGD iteration steps, with other hyper-parameters fixed. Table \ref{tab:pgd_iter} shows that with PGD Iteration steps ranging from 10 to 100, the robust accuracy  of our method changes little.

\subsubsection{Impact of number of attack budgets and mixing coefficients}

\begin{table}[h]
\caption{Robust Accuracy against Classifier-PGD attack under $\ell_{\infty}$ threat model on MNIST dataset, with different attack budgets $\epsilon$ and mixing coefficients $\lambda$. Number of PGD attack steps is set to 40. We mark the best performance for each budget by an underlined and bold value and the second best by bold value.}
\vskip 0.1in
\centering
\begin{tabular}{c c c c c c}
\toprule
$\lambda$ & 0.0 (ODE) & 0.25 &  0.5 & 0.75 & 1.0 (SDE) \\
\midrule 
Standard Acc $(\%)$ & 99.17  & 99.22 & 99.10 & 99.06 & 99.00\\
\midrule
Robust Acc $(\%)$ \\
\quad $\epsilon = 0.10$  & 97.57  & \textbf{97.96} & \underline{\textbf{98.03}} & 97.92 & 97.91\\
\quad $\epsilon = 0.15$ &  96.36 & 96.82 & \textbf{96.98} & \underline{\textbf{97.00}} & 96.93\\
\quad $\epsilon = 0.20$ &  94.06 & 95.17 & \textbf{95.53} & \underline{\textbf{95.69}} & 95.44\\
\quad $\epsilon = 0.25$ & 90.51 & 92.53 & \textbf{93.15} & \underline{\textbf{93.36}} & 92.98\\
\quad $\epsilon = 0.30$ &  88.40 & 88.06 & \textbf{89.24} & \underline{\textbf{89.28}} & 88.89\\
\bottomrule
\end{tabular}
\vskip 0.1in
\label{tab:mnist_budget}
\end{table}

\begin{table}[h]
\caption{Robust Accuracy against Classifier-PGD attack under $\ell_{\infty}$ threat model on CIFAR10 dataset, with different attack budgets $\epsilon$. Number of PGD attack steps is set to 40.}
\vskip 0.1in
\centering
\begin{tabular}{c c c c c}
\toprule
$\epsilon$ & 2/255. &  4/255. & 8/255. & 16/255. \\
\midrule
Robust Accuracy $(\%)$ & 91.35 & 90.79 & 90.03 & 89.24\\
\bottomrule
\end{tabular}
\vskip 0.1in
\label{tab:cifar10_budget}
\end{table}

In this section, we conduct diverse experiments against Classifier-PGD attack on both MNIST and CIFAR10 datasets with different perturbation budgets and different choices of randomness level. Else hyper-parameters are fixed across all experiments. Table \ref{tab:mnist_budget} and Table \ref{tab:cifar10_budget} demonstrate that with the attack budget increasing, robustness performance of our method does not decrease dramatically. Table \ref{tab:mnist_budget} also show that the optimal mixing coefficients are similar across different attack budgets.

\subsubsection{Performance with different classifier architectures}

\begin{table}[h]
\caption{Standard accuracy and robust accuracy against Classifier-PGD attack using different classifier architectures.}
\vskip 0.1in
\centering
\begin{tabular}{c c c c}
\toprule
 Architecture & Nature Accuracy & Standard Accuracy & Robust Accuracy \\
\midrule
ResNet-50  & 96.04 & 93.12 & 90.82 \\
WideResNet-28-10 & 95.19 & 92.41 & 90.08 \\
WideResNet-70-16 & 95.19 & 92.2 & 89.95 \\
\bottomrule
\end{tabular}
\vskip 0.1in
\label{tab:clf-arc}
\end{table}

To evaluate that our method is model-agnostic, we consider three different classifier architectures: ResNet-50, WideResNet-28-10, and WideResNet-70-16. We perform experiments on the CIFAR10 dataset against a Classifier-PGD attack, with $\ell_{\infty}$-norm, $\epsilon=8/255$, 40 PGD steps and other parameters fixed. Table \ref{tab:clf-arc} shows that our method achieves high robust accuracy consistently among all architectures.f
 
\section{Purified Adversarial Examples on the CIFAR10 Dataset}

\begin{figure}
\vskip 0.1in
\centering
\includegraphics[height=10cm,width=10cm]{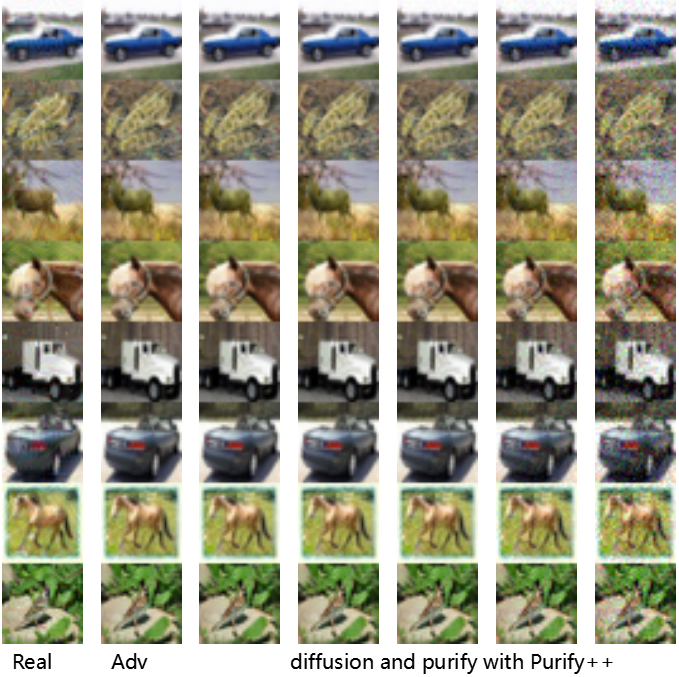}
\caption{Purified Adversarial Examples on the CIFAR10 Dataset.}
\label{fig:cifar10_vis}
\vskip -0.1in
\end{figure}


\end{document}